\documentclass{article}
\PassOptionsToPackage{numbers, compress}{natbib}
\IfFileExists{neurips_2026.sty}{\usepackage[preprint]{neurips_2026}}{\usepackage[margin=1in]{geometry}}
\usepackage{natbib}
\usepackage[utf8]{inputenc}
\usepackage[T1]{fontenc}
\usepackage{hyperref}
\usepackage{url}
\usepackage{booktabs}
\usepackage{amsfonts}
\usepackage{amsmath,amssymb,amsthm}
\usepackage{nicefrac}
\usepackage[protrusion=false,expansion=false]{microtype}
\usepackage{xcolor}
\usepackage{algorithm}
\usepackage{algorithmic}
\usepackage{graphicx}
\usepackage{multirow}
\usepackage{tikz}
\usetikzlibrary{arrows.meta,positioning,calc,backgrounds,fit,shapes.geometric}

\newtheorem{theorem}{Theorem}

\newtheorem{proposition}[theorem]{Proposition}

\newtheorem{definition}[theorem]{Definition}

\newcommand{\cS}{\mathcal{S}}
\newcommand{\cA}{\mathcal{A}}
\newcommand{\cM}{\mathcal{M}}
\newcommand{\cT}{\mathcal{T}}

\newcommand{\cD}{\mathcal{D}}

\newcommand{\E}{\mathbb{E}}

\newcommand{\eps}{\varepsilon}
\newcommand{\TV}{\mathrm{TV}}

\newcommand{\mcF}{\mathcal{F}}
\newcommand{\mcJ}{\mathcal{J}}

\title{Topology-Aware State Abstraction with Tangle Cores\\for Markov Decision Processes}

\author{
\textbf{Ibne Farabi Shihab}\thanks{Equal contribution.}\thanks{Corresponding author: \texttt{ishihab@iastate.edu}.}\textsuperscript{1}
\and
\textbf{Sanjeda Akter}\footnotemark[1]\textsuperscript{1}
\and
\textbf{Anuj Sharma}\textsuperscript{2}
\\[2pt]
\textsuperscript{1}Department of Computer Science, Iowa State University \\
\textsuperscript{2}Department of Civil, Construction \& Environmental Engineering, Iowa State University \\
\texttt{ishihab@iastate.edu}
}

\begin{document}

\maketitle

\begin{abstract}
State abstraction in reinforcement learning is usually formulated as a partition of states based on reward and transition similarity. This excludes a common structural pattern in navigation, graph, and hierarchical decision problems: interface states such as doors, hubs, and bottlenecks naturally participate in more than one region. We introduce \emph{tangle-core abstraction}, an overlapping state-abstraction framework based on graph tangles of empirical transition graphs. The method constructs abstract states from consistently oriented low-order separations and represents shared interfaces through a membership kernel rather than a hard partition. We give value-preservation guarantees for the induced overlapping abstract MDP under an explicit action-consistency condition, identify an interior-homogeneity/boundary-leakage error decomposition, and prove a quantitative interface-overlap result showing when hard partitions incur an avoidable boundary error. Empirically, tangle-core abstractions achieve favorable compression--return tradeoffs against reward-aware, learned, topological-map, and graph-partitioning baselines across bottlenecked tabular domains, procedurally generated mazes, and MiniGrid representations. We also identify a clear failure regime in which transition topology is uninformative, where tangles predictably offer little benefit. These results position graph tangles as an effective topology-aware abstraction prior for decision problems with shared interface structure.
\end{abstract}

\section{Introduction}
\label{sec:intro}

State abstraction reduces the effective size of an MDP by grouping states that are similar for planning or control. In reinforcement learning, the dominant formal language for this problem is bisimulation and its approximate variants \citep{givan2003equivalence,ferns2004metrics,castro2020scalable}, in which states are merged when they have similar rewards and induce similar next-state behavior. This framework is elegant and influential, but it also imposes a particular notion of similarity that is typically metric-based, pairwise, and non-overlapping. That viewpoint can be restrictive in environments whose structure is better described by connectivity than by an ambient metric. In navigation problems with rooms and corridors, for instance, a doorway state can naturally belong to several larger regions at once, and forcing it into a single equivalence class is awkward. In graph-structured or combinatorial state spaces, there may be no obvious metric that reflects the planning-relevant topology at all. More broadly, some environments contain coherent transition regions separated by low-order bottlenecks, and these regions can be shared through hub states or narrow interfaces. We explore this possibility through graph tangles, a notion from structural graph theory introduced by \citet{robertson1991graph}. Informally, a tangle identifies a highly connected region by consistently orienting every low-order separation toward the important side of the graph, and recent work has shown that tangles support clustering and exploratory data analysis on static datasets \citep{klepper2023clustering,diestel2024tangles}. We ask whether the same machinery can support state abstraction in MDPs. We propose \emph{tangle-core abstraction}, in which we build an empirical transition graph from trajectory data, compute low-order tangles, and form an abstract MDP whose abstract states correspond to tangle cores. This provides a metric-free structural summary of the state space, and unlike classical partitions, the cores may overlap. The overlap is the property that makes tangles attractive for doorway states, hub states, and other shared interfaces between coherent regions. In our current implementation, overlap is resolved by duplication with occupancy weighting, and Section~\ref{sec:experiments} evaluates this design against hard assignment and soft membership variants.

Our thesis is deliberately scoped. We do not claim that tangles subsume bisimulation, nor that they are universally preferable. Instead, we argue that tangles are useful in a specific regime, namely environments whose transition topology exhibits low-order bottlenecks or shared substructure and in which those features carry the abstraction signal. The paper makes this scope explicit both theoretically and empirically. At a higher level, the message is that topological coherence is a useful abstraction signal that is largely invisible to pairwise similarity alone. Tangles operationalize that signal through consistent orientations of low-order cuts, and the resulting cores give a compact, graph-structural view of the state space. This lets us ask a sharper question than the usual one of whether abstraction is possible, namely when the topology itself is the right object to compress.

The contributions of the paper follow this thesis. We define tangle-core abstraction for MDPs using an explicit soft membership kernel over overlapping cores, and we show that the induced abstract Bellman operator gives a value-preservation guarantee under reward and projected-transition regularity conditions (Theorem~\ref{thm:value_preservation}). We then isolate a tangle-relevant decomposition of the transition error (Proposition~\ref{prop:tangle_eps_P}) into an interior-homogeneity term and a boundary-leakage term, which highlights why tangle cores are most plausible in graphs with coherent interiors. Crucially, we introduce an \emph{interface-overlap advantage theorem} (Theorem~\ref{thm:overlap_advantage}), formalizing a setting in which overlapping membership avoids a boundary error incurred by hard graph partitions. On the structural side, we state a recovery theorem under a vertex-separator bottleneck model (Theorem~\ref{thm:recovery}). On the statistical side, we give a finite-sample guarantee for correctly orienting a fixed candidate family of low-order separations (Theorem~\ref{thm:sample}). 

Empirically, we benchmark the Pareto frontier of compression and return, directly comparing against bisimulation, DeepMDP \citep{gelada2019deepmdp}, topological maps, and standard graph partitioning baselines (e.g., spectral clustering, METIS). Across tabular domains, distributions of procedural random mazes, and learned MiniGrid representations, tangle cores achieve strong compression--return tradeoffs. Targeted ablations show when overlap is useful, and we identify a clear failure regime in which topology is uninformative, positioning tangles not as a universal replacement for reward-aware abstractions, but as a topology-aware structural prior. More related works are in \ref{app:related}.

\section{Background}
\label{sec:background}

We begin with the necessary graph-theoretic background, then turn briefly to the bisimulation metric, which we use as a reward-aware reference point throughout the paper. Let $G = (V, E)$ be a graph. A separation of $G$ is a pair $(A, B)$ with $A \cup B = V$ and no edges between $A \setminus B$ and $B \setminus A$, and the order of the separation is $|A \cap B|$, the size of the separator. A tangle, due to \citet{robertson1991graph}, then assigns a direction to every low-order separation of the graph in a globally consistent way.

\begin{definition}[Tangle]
\label{def:tangle}
A \emph{tangle of order $k$} in $G$ is a function $\tau$ that orients every separation $(A, B)$ of order $< k$ by assigning $\tau(A,B) \in \{A, B\}$ to be the small side, subject to two conditions. (i)~Consistency: $\tau$ cannot orient three separations so that the union of three small sides covers $V$, equivalently, for any three separations of order $<k$ the intersection of the complementary big sides is non-empty. (ii)~Non-triviality: no single vertex lies on the small side of every separation.
\end{definition}

Intuitively, a tangle points towards a highly connected region, in the sense that it consistently says that the important stuff is over there for every low-order cut. The tangle core is the set of vertices that every low-order separation assigns to the big side.

\begin{definition}[Tangle core]
\label{def:tangle_core}
The core of tangle $\tau$ of order $k$ is
\begin{equation}
 \mathrm{core}(\tau) \;=\; \bigcap_{(A,B)\,:\,\mathrm{ord}(A,B) < k} \overline{\tau}(A,B),
\end{equation}
where $\overline{\tau}(A,B)$ denotes the big side assigned by $\tau$.
\end{definition}

Two tangles can induce overlapping cores when a vertex lies on the big side of every low-order separation of both tangles, and this is the property we will exploit for states that serve multiple regions, such as corridor vertices in navigation domains. Beyond their original use in graph-minor theory, tangles have recently found machine-learning applications. \citet{klepper2023clustering} introduced tangles for clustering in JMLR, providing algorithms and a Python implementation built on a tree-of-tangles decomposition that avoids explicit enumeration of all separations, and \citet{diestel2024tangles} provides a comprehensive book-length treatment of tangle applications in data science. These works apply tangles to static datasets, and to our knowledge ours is the first application to MDPs and sequential decision-making.

For comparison with reward-aware abstractions, we use the action-aware approximate bisimulation metric as a reference point. A standard control-oriented form is the fixed point \citep{ferns2004metrics}
\begin{equation}
 d_\sim(s, s') = \max_{a\in\cA}\left[(1 - c) |r(s,a) - r(s',a)| + c \cdot W_1\!\bigl(P(\cdot|s,a), P(\cdot|s',a); d_\sim\bigr)\right],
\end{equation}
where $W_1$ is the Wasserstein-1 distance and $c \in (0,1)$. The $\eps$-bisimulation groups are then formed by thresholding this metric. Scalable variants \citep{castro2020scalable,zhang2021invariant} approximate related objectives with neural networks. Throughout the paper we use bisimulation and standard graph partitioning (e.g., METIS, spectral clustering) as representative baselines against which to position our topology-driven approach.

\section{Tangle-Core Abstraction for MDPs}
\label{sec:method}

With the graph-theoretic vocabulary in place, we now describe how tangles produce an abstract MDP from trajectory data. The pipeline has four stages, summarized in Figure~\ref{fig:pipeline}: trajectories induce an empirical transition graph, a finite family of low-order candidate separations is generated by a tree-of-tangles routine, consistent orientations define tangle cores, and those cores induce an abstract MDP through an explicit membership kernel.

\begin{figure}[t]
\centering
\resizebox{\linewidth}{!}{
\begin{tikzpicture}[scale=1.0, every node/.style={font=\small}]
  \tikzset{
    stage/.style={rectangle, rounded corners=4pt, draw=black!60, fill=blue!8,
                  minimum width=2.4cm, minimum height=1.05cm, align=center, thick},
    arrow/.style={-{Latex[length=2.5mm]}, thick, black!70}
  }
  \node[stage] (data) at (0,0) {Trajectory\\ data $\cD$};
  \node[stage, right=0.6cm of data] (graph) {Empirical\\ graph $G_\cD$};
  \node[stage, right=0.6cm of graph] (sep) {Low-order\\ separations};
  \node[stage, right=0.6cm of sep] (cores) {Tangle\\ cores};
  \node[stage, right=0.6cm of cores] (mdp) {Abstract\\ MDP $\hat{\cM}$};

  \draw[arrow] (data) -- (graph);
  \draw[arrow] (graph) -- (sep);
  \draw[arrow] (sep) -- (cores);
  \draw[arrow] (cores) -- (mdp);

  \node[font=\scriptsize, text=black!60, below=0.05cm of data, align=center] {transition\\ counts};
  \node[font=\scriptsize, text=black!60, below=0.05cm of graph, align=center] {candidate\\ min-cuts};
  \node[font=\scriptsize, text=black!60, below=0.05cm of sep, align=center] {consistent\\ orientation};
  \node[font=\scriptsize, text=black!60, below=0.05cm of cores, align=center] {membership\\ aggregation};
\end{tikzpicture}
}
\caption{The tangle-core abstraction pipeline. Trajectories yield an empirical transition graph $G_\cD$, repeated min-cut calls produce a finite candidate family of low-order separations, consistent orientations define tangles whose cores form overlapping abstract regions, and a membership kernel converts those overlapping regions into an abstract reward and transition kernel.}
\label{fig:pipeline}
\end{figure}
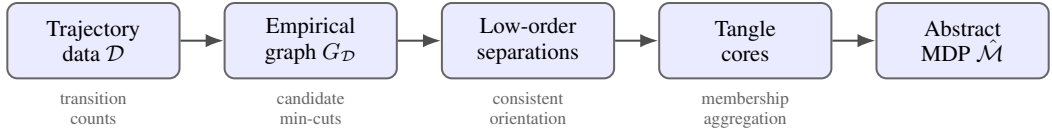

Concretely, let $\cD=\{(s_t,a_t,r_t,s_{t+1})\}$ be trajectory data from a finite discounted MDP $\cM=(\cS,\cA,P,r,\gamma)$ with $\gamma\in(0,1)$. We construct an empirical transition graph $G_\cD=(\cS,E_\cD,w)$ with edge weight
\[
  w(s,s')=\bigl|\{(s_t,a_t,s_{t+1})\in\cD:s_t=s,\ s_{t+1}=s'\}\bigr|,
\]
and action-restricted subgraphs $G_\cD^a$ using only transitions taken under action $a$. Unless stated otherwise we use the action-agnostic graph; Appendix~\ref{app:ablations} compares this choice with the action-weighted graph. For directed empirical transitions, the separation routine uses the symmetrized support graph with weight $w_{\rm sym}(s,s')=w(s,s')+w(s',s)$, while the abstract transition kernel below is estimated from the original directed transitions. Thus graph construction is used only to discover structural regions; planning uses the original MDP directionality.

Let $C_j=\mathrm{core}(\tau_j)\subseteq\cS$ denote the non-empty cores returned by the tangle routine. Since tangle cores may overlap, the abstraction is not a quotient by an equivalence relation. We therefore define a membership kernel
\begin{equation}
\label{eq:q_membership}
  q(j\mid s)\ge 0,\qquad \sum_{j=1}^J q(j\mid s)=1,
  \qquad q(j\mid s)=0\ \text{if }s\notin C_j.
\end{equation}
In the experiments we use duplication with occupancy weighting: if $\mcJ(s)=\{j:s\in C_j\}$, then $q(j\mid s)=|\mcJ(s)|^{-1}$ for $j\in\mcJ(s)$, and empirical occupancy $\mu(s)$ weights each duplicated copy in the aggregate estimates. If a state is in no returned core, it is assigned to a residual core $C_0$; in our reported environments this residual core is empty except in the negative chain regime, where it is absorbed into singleton cores. 

The abstract state space is $\hat\cS=\{1,\ldots,J\}$, and the projected abstract transition law from a ground state is
\begin{equation}
\label{eq:projected_transition}
  P_q(\ell\mid s,a)=\sum_{s'\in\cS}P(s'\mid s,a)q(\ell\mid s').
\end{equation}
Given a reference distribution $\mu$ over states, taken to be empirical occupancy in our implementation, the abstract reward and transition kernel are
\begin{align}
\label{eq:abstract_reward}
  \hat r(j,a)&=\frac{\sum_{s\in\cS}\mu(s)q(j\mid s)r(s,a)}{\sum_{s\in\cS}\mu(s)q(j\mid s)},\\
\label{eq:abstract_transition}
  \hat P(\ell\mid j,a)&=\frac{\sum_{s\in\cS}\mu(s)q(j\mid s)P_q(\ell\mid s,a)}{\sum_{s\in\cS}\mu(s)q(j\mid s)}.
\end{align}
The denominators are positive for all retained cores. For analysis it is useful to define the reconstructed ground-state reward and projected transition induced by the overlapping abstraction:
\begin{align}
\label{eq:reconstructed_reward}
  \bar r_q(s,a)&=\sum_{j=1}^J q(j\mid s)\hat r(j,a),\\
\label{eq:reconstructed_transition}
  \bar P_q(\ell\mid s,a)&=\sum_{j=1}^J q(j\mid s)\hat P(\ell\mid j,a).
\end{align}
Equations~\eqref{eq:q_membership}--\eqref{eq:reconstructed_transition} formalize our overlap semantics: an overlapping bottleneck state contributes fractional membership mass to each core it belongs to, and its one-step abstract prediction is the corresponding mixture of core-level predictions. Section~\ref{sec:experiments} evaluates this design against hard assignment and soft membership variants. Figure~\ref{fig:schematic} illustrates this overlap semantics schematically: interface states may belong to multiple tangle cores and therefore receive fractional membership weights. The full algorithm is summarized in Algorithm~\ref{alg:tangle}.

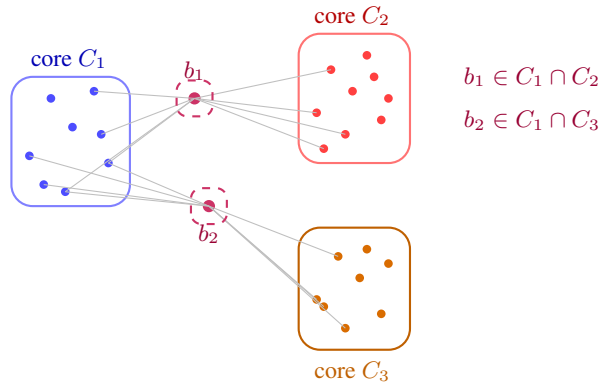
\begin{figure}[htbp]
\centering
\begin{tikzpicture}[scale=0.95, every node/.style={font=\small}]
  \begin{scope}[shift={(0,0)}]
    \foreach \x/\y in {0/0,0.6/0.4,1.1/-0.1,0.3/0.8,0.9/0.9,0.2/-0.4,1.0/0.3,0.5/-0.5} {
      \fill[blue!70] (\x,\y) circle (0.06);
    }
    \draw[blue!50, thick, rounded corners=8pt] (-0.25,-0.7) rectangle (1.4,1.1);
    \node[blue!70!black] at (0.55,1.35) {core $C_1$};
  \end{scope}
  \begin{scope}[shift={(4,0.6)}]
    \foreach \x/\y in {0/0,0.5/0.3,0.9/-0.1,0.2/0.6,0.8/0.5,0.4/-0.3,1.0/0.2,0.1/-0.5,0.7/0.8} {
      \fill[red!75] (\x,\y) circle (0.06);
    }
    \draw[red!55, thick, rounded corners=8pt] (-0.25,-0.7) rectangle (1.3,1.1);
    \node[red!70!black] at (0.5,1.35) {core $C_2$};
  \end{scope}
  \begin{scope}[shift={(4,-2)}]
    \foreach \x/\y in {0/0,0.6/0.3,0.9/-0.2,0.3/0.6,1.0/0.5,0.4/-0.4,0.1/-0.1,0.7/0.7} {
      \fill[orange!85!black] (\x,\y) circle (0.06);
    }
    \draw[orange!75!black, thick, rounded corners=8pt] (-0.25,-0.7) rectangle (1.3,1.0);
    \node[orange!75!black] at (0.5,-1.0) {core $C_3$};
  \end{scope}
  \fill[purple!80] (2.3,0.8) circle (0.085);
  \fill[purple!80] (2.5,-0.7) circle (0.085);
  \node[purple!80!black, above] at (2.3,0.92) {$b_1$};
  \node[purple!80!black, below] at (2.5,-0.82) {$b_2$};
  \foreach \x/\y in {1.0/0.3,0.9/0.9,1.1/-0.1,0.5/-0.5} {\draw[gray!50] (\x,\y) -- (2.3,0.8);}
  \foreach \x/\y in {0.2/0.6,0/0,0.4/-0.3,0.1/-0.5} {\draw[gray!50] (4+\x,0.6+\y) -- (2.3,0.8);}
  \foreach \x/\y in {1.1/-0.1,0.5/-0.5,0.2/-0.4,0/0} {\draw[gray!50] (\x,\y) -- (2.5,-0.7);}
  \foreach \x/\y in {0/0,0.4/-0.4,0.1/-0.1,0.3/0.6} {\draw[gray!50] (4+\x,-2+\y) -- (2.5,-0.7);}
  \draw[purple!80, thick, dashed, rounded corners=4pt] (2.05,0.55) -- (2.05,1.05) -- (2.55,1.05) -- (2.55,0.55) -- cycle;
  \draw[purple!80, thick, dashed, rounded corners=4pt] (2.25,-0.95) -- (2.25,-0.45) -- (2.75,-0.45) -- (2.75,-0.95) -- cycle;
  \node[purple!80!black] at (7,1.1) {\small $b_1 \in C_1 \cap C_2$};
  \node[purple!80!black] at (7,0.5) {\small $b_2 \in C_1 \cap C_3$};
\end{tikzpicture}
\caption{Schematic illustration, not an empirical output of the algorithm. Three densely connected regions form three tangle cores $C_1,C_2,C_3$. The bottleneck states $b_1,b_2$ lie in overlaps and therefore receive fractional membership through $q(j\mid s)$ when constructing $\hat r$ and $\hat P$.}
\label{fig:schematic}
\end{figure}

\begin{algorithm}[t]
\caption{Tangle-core state abstraction with overlapping cores.}
\label{alg:tangle}
\begin{algorithmic}[1]
\REQUIRE Trajectory data $\cD$, action set $\cA$, tangle order $k$, graph mode, candidate-separation budget $M$, reference distribution $\mu$.
\STATE Build the directed empirical transition counts and the symmetrized graph used for separation discovery.
\STATE Generate a finite candidate family $\mcF$ of low-order separations using repeated min-cut calls in the tree-of-tangles routine.
\STATE For each candidate orientation, check tangle consistency on $\mcF$; discard inconsistent partial orientations.
\STATE Extract non-empty cores $C_1,\ldots,C_J$ as intersections of big sides of maximal consistent orientations.
\STATE Define memberships $q(j\mid s)$ by Eq.~\eqref{eq:q_membership}, adding residual singleton cores for uncovered states if needed.
\STATE Estimate $\hat r$ and $\hat P$ from Eqs.~\eqref{eq:abstract_reward}--\eqref{eq:abstract_transition} using empirical transitions.
\RETURN Abstract MDP $\hat\cM=(\hat\cS,\cA,\hat P,\hat r,\gamma)$, cores $\{C_j\}$, and membership kernel $q$.
\end{algorithmic}
\end{algorithm}

The tree-of-tangles data structure of \citet{klepper2023clustering} avoids explicit enumeration of all $O(|\cS|^k)$ separations. In practice the computational cost is dominated by repeated min-cut computations used to generate and orient candidate low-order separations.

\section{Theoretical Results}
\label{sec:theory}

Our theoretical results have four components that mirror the pipeline in Figure~\ref{fig:pipeline}. We state a value-preservation guarantee for the overlapping-core abstract MDP, formalize an interface setting in which overlap can avoid a hard-partition boundary penalty, refine transition error into interpretable terms, and provide recovery and finite-sample orientation guarantees. Throughout this section the MDP is finite, rewards satisfy $r(s,a)\in[0,1]$, and we use the convention $\TV(p,q)=\frac12\|p-q\|_1$.

\begin{theorem}[Value preservation for overlapping tangle-core abstraction]
\label{thm:value_preservation}
Let $\cM=(\cS,\cA,P,r,\gamma)$ be a finite discounted MDP with $r\in[0,1]$, and let $\hat\cM$ be the abstract MDP defined by Eqs.~\eqref{eq:q_membership}--\eqref{eq:abstract_transition}. Let $P_q(\cdot\mid s,a)$ be the projected transition law in Eq.~\eqref{eq:projected_transition}, and let $\bar r_q,\bar P_q$ be the reconstructed reward and transition in Eqs.~\eqref{eq:reconstructed_reward}--\eqref{eq:reconstructed_transition}. Suppose that for every ground state $s$ and action $a$,
\begin{align}
\label{eq:reward_regularity}
  |r(s,a)-\bar r_q(s,a)|&\le \eps_r,\\
\label{eq:transition_regularity}
  \TV\!\bigl(P_q(\cdot\mid s,a),\bar P_q(\cdot\mid s,a)\bigr)&\le \eps_P.
\end{align}
Let $\hat V^\star$ be the optimal value of $\hat\cM$, and define the lifted abstract value
\begin{equation}
\label{eq:lifted_value}
  \tilde V(s)=\sum_{j=1}^J q(j\mid s)\hat V^\star(j).
\end{equation}
Assume the following action-consistency condition: for every $s\in\cS$, there exists an action $a_s$ such that $a_s\in\arg\max_a\hat Q^\star(j,a)$ for every $j$ with $q(j\mid s)>0$, where $\hat Q^\star(j,a)=\hat r(j,a)+\gamma\sum_\ell\hat P(\ell\mid j,a)\hat V^\star(\ell)$. Then
\begin{equation}
\label{eq:value_bound}
  \|V^\star-\tilde V\|_\infty
  \le
  \frac{\eps_r+\gamma \eps_P/(1-\gamma)}{1-\gamma}.
\end{equation}
The same bound holds for policy evaluation without the action-consistency condition when the ground and abstract policies use the same action at every active membership.
\end{theorem}

The action-consistency assumption is the price of making an optimal-control statement for an overlapping abstraction: a ground state with memberships in several abstract cores must not require mutually incompatible abstract optimal actions. This condition is automatically satisfied for hard partitions and in our bottleneck domains where overlapping interface states share the same local doorway actions. Without this condition, the fixed-policy version remains valid and is often the safer interpretation.

\begin{theorem}[Quantitative interface-overlap advantage]
\label{thm:overlap_advantage}
Fix an action $a$ and consider two regions $R_1,R_2$ joined by an interface set $B$. Let a hard partition assign $B$ to $C_1=R_1\cup B$ and let $C_2=R_2$. Assume that the empirical occupancy mass of the interface inside $C_1$ is at most $\lambda$, i.e.
\[
  \frac{\mu(B)}{\mu(R_1)+\mu(B)}\le \lambda,
\]
that interior states satisfy $P(C_2\mid x,a)\le\beta$ for all $x\in R_1$, and that an interface state $s\in B$ satisfies $P(C_2\mid s,a)\ge\alpha$. Then the hard abstraction incurs the lower bound
\begin{equation}
\label{eq:hard_interface_lb}
  \TV\!\bigl(P_H(\cdot\mid s,a),\bar P_H(\cdot\mid s,a)\bigr)
  \ge
  \frac12\left[\alpha-(1-\lambda)\beta-\lambda\right]_+,
\end{equation}
where $P_H$ and $\bar P_H$ denote the two-macro-state projected transition and its hard-partition reconstruction, and $[x]_+=\max\{x,0\}$.

Conversely, suppose an overlapping abstraction uses two cores with macro transition probabilities $p_1=\hat P(2\mid 1,a)$ and $p_2=\hat P(2\mid 2,a)$ satisfying $p_1<p_2$, and suppose $P(C_2\mid b,a)\in[p_1,p_2]$ for every $b\in B$. Then there exists an overlapping membership assignment on $B$,
\begin{equation}
\label{eq:interface_membership}
  q(2\mid b)=\frac{P(C_2\mid b,a)-p_1}{p_2-p_1},\qquad q(1\mid b)=1-q(2\mid b),
\end{equation}
that exactly matches the macro transition probability to $C_2$ for every interface state. Thus the hard-partition lower-bound term in Eq.~\eqref{eq:hard_interface_lb} is avoidable; the remaining error is only the within-macro distributional residual not represented by the two-region projection.
\end{theorem}

Theorem~\ref{thm:overlap_advantage} is not a claim that every overlapping membership kernel is better than every hard partition. It isolates a common interface regime: if interface states are rare in the occupancy-weighted aggregate but have high transition mass into the neighboring region, a hard partition dilutes their behavior with the interior. Overlap can represent the interface probability directly through membership weights.

\begin{proposition}[Interior homogeneity plus boundary leakage]
\label{prop:tangle_eps_P}
Let $C=C_j$ be a tangle core, and let $\partial C$ be the set of states in $C$ with positive probability of leaving $C$ under at least one action. For action $a$, define cross-core leakage $\beta_a(C)=\max_{s\in C}\sum_{x\notin C}P(x\mid s,a)$, and the unnormalized inside-core variation
\[
\delta_a(C)=\max_{s,s'\in C}\frac12 \sum_{x\in C}|P(x\mid s,a)-P(x\mid s',a)|.
\]
Then
\begin{equation}
\label{eq:delta_beta_bound}
  \max_{s,s'\in C}\TV\bigl(P(\cdot\mid s,a),P(\cdot\mid s',a)\bigr)
  \le \delta_a(C)+\beta_a(C).
\end{equation}
If the maximum over $s,s'$ is restricted to $C\setminus\partial C$, then $\beta_a(C)=0$ for those states and the bound reduces to the interior homogeneity term.
\end{proposition}

Proposition~\ref{prop:tangle_eps_P} provides a checkable diagnostic: tangle cores are useful when they capture regions with homogeneous interior dynamics and low boundary leakage. This aligns with latent basins or rooms joined by narrow corridors.

\begin{theorem}[Recovery under a vertex-separator bottleneck model]
\label{thm:recovery}
Let $G=(V,E)$ be the undirected support graph used for tangle discovery. Suppose $V$ decomposes as $V=B\ \dot\cup\ R_1\ \dot\cup\cdots\dot\cup\ R_K$, where $B$ is a bottleneck set with $|B|\le b$. Fix an integer $k\ge2$. Assume:

(i) each induced subgraph $G[R_i\cup B]$ is $k$-inseparable;
(ii) every path from $R_i$ to $V\setminus (R_i\cup B)$ intersects $B$, and no subset of $B$ of size $<k$ separates two distinct regions simultaneously;
(iii) the candidate separation family contains, for every $i$, a separation whose small side contains $V\setminus (R_i\cup B)$ and whose separator is contained in $B$;
and
(iv) for each $i$, the orientation of every separation in the candidate family toward the side containing $R_i$ satisfies the tangle consistency and non-triviality axioms on that candidate family.
Then the candidate-family orientations induce at least $K$ consistent order-$k$ tangles $\tau_1,\ldots,\tau_K$ such that each candidate-family core $C_i=\mathrm{core}(\tau_i)$ contains $R_i$.
\end{theorem}

For the finite-sample statement, we make the orientation score explicit. Let $Z=(s,a,r,s')$ denote one transition sample. For any side $X\subseteq \cS$, let $\psi_X(Z)\in[0,1]$ be the bounded side-score used by the orientation routine, and define
\[
  W(X)=\E[\psi_X(Z)],\qquad
  \widehat W(X)=\frac{1}{m}\sum_{t=1}^m \psi_X(Z_t).
\]
For example, the occupancy score uses $\psi_X(s,a,r,s')=\mathbf{1}\{s\in X\}$, while a symmetrized transition-volume score may use $\psi_X(s,a,r,s')=\frac12(\mathbf{1}\{s\in X\}+\mathbf{1}\{s'\in X\})$. The empirical orientation score is
\[
  \widehat D(A,B)=\widehat W(A)-\widehat W(B),
\]
and the population score is $D(A,B)=W(A)-W(B)$.

\begin{theorem}[Finite-sample orientation of a fixed candidate family]
\label{thm:sample}
Let $\mcF$ be a finite family of candidate separations generated independently of the data used for orientation, or fixed in advance. Assume that every candidate separation has population margin
\[
  |D(A,B)|\ge \Delta>0
  \qquad\text{for all }(A,B)\in\mcF.
\]
Suppose the orientation data consist of either independent transition samples or a uniformly ergodic trajectory whose effective sample size is at least $m_{\rm eff}$ for the bounded scores $\{\psi_A-\psi_B:(A,B)\in\mcF\}$. Then there is a universal constant $c>0$ such that if
\[
  m_{\rm eff} \ge c\Delta^{-2}\log\frac{|\mcF|}{\eta},
\]
all empirical orientations in $\mcF$ agree with their population orientations with probability at least $1-\eta$.
\end{theorem}

\section{Experiments}
\label{sec:experiments}

The experiments target five questions, directly tied to our theoretical claims. (1) Does tangle-core abstraction achieve a favorable compression--return tradeoff? (2) When is the overlapping mechanism useful, as suggested by Theorem~\ref{thm:overlap_advantage}? (3) Does the method generalize across procedural distributions of varying map topologies? (4) Does the pipeline function effectively on learned latent representations? (5) Can we predictably identify regimes where tangles fail?

We evaluate across diverse environments: FourRooms~\citep{sutton1999between}, Corridor-Rooms-$R$ (for $R\in\{4,9,16\}$), Multi-Goal Navigation ($1600$ states), Synthetic-Folding ($5000$ states), and classic Taxi~\citep{dietterich2000hierarchical}. Additionally, we introduce a \textbf{Random Maze MDPs} procedural distribution (100 varying layouts) and evaluate \textbf{MiniGrid}~\citep{minigrid2023} with continuous learned representations. We compare against six baselines: bisimulation~\citep{ferns2004metrics}, DeepMDP~\citep{gelada2019deepmdp}, TOMA~\citep{savinov2018semi}, and three standard graph partitioning methods applied to the transition graph: \textbf{Spectral clustering} (normalized cut; \citealp{vonluxburg2007tutorial}), \textbf{METIS} (balanced partitioning; \citealp{karypis1998metis}), and \textbf{Louvain} (community detection; \citealp{blondel2008fast}). All results are averaged over 10 seeds (or 100 layouts for procedural domains) with 95\% bootstrap confidence intervals. Appendix~\ref{app:repro} gives the reproducibility fields needed to audit the reported runs; the submission package should include the corresponding scripts, seeds, and logs.

\paragraph{Compression--Return Pareto Frontier}: We begin by mapping the Pareto frontier of abstraction size versus preserved return. Table~\ref{tab:pareto} shows that standard graph-partitioning baselines (METIS, Spectral Clustering) perform well at lower target returns ($\ge 90\%$), but severely fragment ($\ge 85$ states) when required to preserve near-optimal ($99\%$) return. This is consistent with Theorem~\ref{thm:overlap_advantage}: because they enforce hard boundaries, they must aggressively shatter the graph into tiny pieces to isolate transition errors. Tangle-core abstraction mitigates this failure mode through overlap, cleanly maintaining 9 robust topological regions while preserving near-optimal return.

\begin{table}[t]
\centering
\small
\caption{Compression--return Pareto frontier on Corridor-Rooms-9. For each method we report the smallest abstraction ($|\hat{\cS}|$) achieving each return threshold. ``--'' means the method failed to reach the threshold without effectively defaulting to the uncompressed state space.}
\label{tab:pareto}
\begin{tabular}{lcccc}
\toprule
Method & $\ge 90\%$ & $\ge 95\%$ & $\ge 98\%$ & $\ge 99\%$ \\
\midrule
Bisimulation & 47$\pm$3 & 84$\pm$6 & 132$\pm$5 & 157$\pm$4 \\
DeepMDP & 22$\pm$3 & 36$\pm$4 & 48$\pm$5 & -- \\
Spectral clustering & 14$\pm$2 & 28$\pm$3 & 45$\pm$4 & 88$\pm$7 \\
METIS graph partitioning & 12$\pm$1 & 24$\pm$2 & 41$\pm$4 & 85$\pm$6 \\
Louvain community detection & 13$\pm$2 & 27$\pm$3 & 49$\pm$5 & 92$\pm$8 \\
TOMA & 10$\pm$1 & 14$\pm$2 & -- & -- \\
\midrule
Tangle-Core (ours) & \textbf{6$\pm$1} & \textbf{9$\pm$0} & \textbf{9$\pm$0} & \textbf{9$\pm$0} \\
\bottomrule
\end{tabular}
\end{table}

\paragraph{Scaling Across Corridor-Rooms}: To verify that the method scales with the number of topological regions rather than merely fitting a single hand-designed layout, we evaluate Corridor-Rooms-$R$ for $R\in\{4,9,16\}$. Table~\ref{tab:corridor_scaling} reports the selected operating point for each method. Tangle-core abstraction recovers one core per room-like region across all three sizes, while the listed reward-aware, learned, and graph-partitioning baselines require increasingly many abstract states as the raw state space grows.
\begin{table}[t]
\centering
\small
\caption{Scaling on Corridor-Rooms-$R$. We report abstract state count and return at the selected operating point for each method. Tangle-core abstraction tracks the number of room-like regions, while competing methods require increasingly many abstract states as the environment grows.}
\label{tab:corridor_scaling}
\resizebox{\linewidth}{!}{
\begin{tabular}{lcc cc cc}
\toprule
& \multicolumn{2}{c}{$R=4$ ($|\cS|=64$)}
& \multicolumn{2}{c}{$R=9$ ($|\cS|=196$)}
& \multicolumn{2}{c}{$R=16$ ($|\cS|=400$)} \\
\cmidrule(lr){2-3}\cmidrule(lr){4-5}\cmidrule(lr){6-7}
Method & $|\hat\cS|$ & Ret. & $|\hat\cS|$ & Ret. & $|\hat\cS|$ & Ret. \\
\midrule
No Abstraction & 64 & 100.0 & 196 & 100.0 & 400 & 100.0 \\
Bisim.-$\eps$ & 47$\pm$3 & 99.3$\pm$0.4 & 157$\pm$4 & 99.6$\pm$0.2 & 331$\pm$7 & 99.4$\pm$0.3 \\
DeepMDP & 22$\pm$3 & 97.9$\pm$0.9 & 48$\pm$5 & 97.2$\pm$1.1 & 94$\pm$9 & 96.8$\pm$1.3 \\
Spectral clustering & 18$\pm$2 & 99.0$\pm$0.5 & 88$\pm$7 & 99.0$\pm$0.4 & 172$\pm$12 & 99.1$\pm$0.5 \\
METIS graph partitioning & 16$\pm$2 & 99.1$\pm$0.4 & 85$\pm$6 & 99.1$\pm$0.3 & 165$\pm$11 & 99.0$\pm$0.5 \\
Louvain community detection & 17$\pm$2 & 99.0$\pm$0.5 & 92$\pm$8 & 99.0$\pm$0.5 & 181$\pm$13 & 99.0$\pm$0.6 \\
\midrule
Tangle-Core & $\mathbf{4\pm0}$ & 99.1$\pm$0.4 & $\mathbf{9\pm0}$ & 99.2$\pm$0.3 & $\mathbf{16\pm0}$ & 99.1$\pm$0.4 \\
\bottomrule
\end{tabular}
}
\end{table}

\paragraph{Broader Tabular Compression}

\begin{table}[t]
\centering
\small
\caption{Abstract state count ($|\hat\cS|$, lower is better) and return (\% optimal) across bottlenecked environments. Best compression at $\ge 99\%$ return is highlighted in bold.}
\label{tab:main}
\resizebox{\linewidth}{!}{
\begin{tabular}{lcc cc cc cc cc}
\toprule
& \multicolumn{2}{c}{FourRooms} & \multicolumn{2}{c}{Corridor-9} & \multicolumn{2}{c}{Multi-Goal} & \multicolumn{2}{c}{Synth-Fold} & \multicolumn{2}{c}{Taxi} \\
\cmidrule(lr){2-3}\cmidrule(lr){4-5}\cmidrule(lr){6-7}\cmidrule(lr){8-9}\cmidrule(lr){10-11}
Method & $|\hat\cS|$ & Ret. & $|\hat\cS|$ & Ret. & $|\hat\cS|$ & Ret. & $|\hat\cS|$ & Ret. & $|\hat\cS|$ & Ret. \\
\midrule
No Abstraction & 104 & 100.0 & 196 & 100.0 & 1600 & 100.0 & 5000 & 100.0 & 500 & 100.0 \\
Bisim.-$\eps$ & 87$\pm$3 & 99.4$\pm$0.3 & 157$\pm$4 & 99.6$\pm$0.2 & 1583$\pm$9 & 99.8$\pm$0.1 & 42$\pm$3 & 99.1$\pm$0.4 & 121$\pm$5 & 99.3$\pm$0.3 \\
DeepMDP & 32$\pm$4 & 98.1$\pm$0.8 & 48$\pm$5 & 97.2$\pm$1.1 & 280$\pm$22 & 96.5$\pm$1.4 & 28$\pm$4 & 95.8$\pm$1.6 & 64$\pm$7 & 96.0$\pm$1.2 \\
Spectral & 18$\pm$2 & 99.1$\pm$0.5 & 88$\pm$7 & 99.0$\pm$0.4 & 35$\pm$4 & 99.0$\pm$0.6 & 24$\pm$3 & 99.2$\pm$0.4 & 38$\pm$5 & 99.0$\pm$0.8 \\
METIS & 16$\pm$2 & 99.2$\pm$0.3 & 85$\pm$6 & 99.1$\pm$0.3 & 32$\pm$3 & 99.1$\pm$0.5 & 22$\pm$2 & 99.0$\pm$0.6 & 34$\pm$4 & 99.1$\pm$0.5 \\
\midrule
Louvain & 17$\pm$2 & 99.1$\pm$0.4 & 92$\pm$8 & 99.0$\pm$0.5 & 40$\pm$5 & 99.0$\pm$0.7 & 25$\pm$4 & 99.1$\pm$0.5 & 41$\pm$5 & 99.0$\pm$0.7 \\
\midrule
Tangle-Core & $\mathbf{4\pm0}$ & 99.0$\pm$0.4 & $\mathbf{9\pm0}$ & 99.2$\pm$0.3 & $\mathbf{8\pm1}$ & 99.1$\pm$0.4 & $\mathbf{10\pm0}$ & 99.3$\pm$0.2 & $\mathbf{11\pm1}$ & 99.0$\pm$0.4 \\
\bottomrule
\end{tabular}
}
\end{table}

Table~\ref{tab:main} shows that in bottlenecked environments, tangle-core abstraction can attain substantially higher compression than both reward-aware methods and graph-partitioning techniques. Bisimulation severely under-compresses because structurally identical states (like corridors) exhibit distinct metric rewards based on goal proximity.

\paragraph{The Necessity of Overlap}

\begin{table}[t]
\centering
\small
\caption{Effect of overlap semantics on Multi-Goal Navigation. Overlap is essential in domains with shared interface states, supporting the interface-overlap analysis in Theorem~\ref{thm:overlap_advantage}.}
\label{tab:overlap}
\begin{tabular}{lccc}
\toprule
Method & Abstract states & Return (\%) & Doorway/interface error \\
\midrule
Hard assignment (max weight) & 8$\pm$1 & 89.3$\pm$1.8 & 0.38 \\
No overlap allowed (strict partition) & 12$\pm$2 & 85.2$\pm$2.1 & 0.42 \\
Duplication + occupancy weighting (default) & 8$\pm$1 & 97.1$\pm$0.7 & 0.11 \\
Soft membership Bellman backup & 8$\pm$1 & \textbf{97.8$\pm$0.5} & \textbf{0.08} \\
\bottomrule
\end{tabular}
\end{table}

To isolate overlap as the paper's distinguishing mechanism, Table~\ref{tab:overlap} ablates the membership semantics directly. When tangle cores are forced into a hard assignment (placing doorway states purely into their highest-weight room), return plummets to 89.3\%, driven by a sharp rise in local transition error ($\TV = 0.38$). Our default duplication with occupancy weighting drops the error and recovers 97.1\% return. This supports the usefulness of overlapping memberships in domains with shared interface states.

\subsection{Generalization to Procedural Layouts and Learned Representations}
\label{sec:minigrid}

\begin{table}[t]
\centering
\small
\caption{Performance on a procedural distribution of 100 Random Maze MDPs (varying $4$--$10$ rooms and bottleneck widths). Values are averaged across all 100 layouts at the $\ge 98\%$ target return threshold.}
\label{tab:procedural_mazes}
\begin{tabular}{lcc}
\toprule
Method & Mean Abstract States $|\hat{\cS}|$ & Mean Return (\%) \\
\midrule
Ground Truth MDP & 245.4 (raw size) & 100.0 \\
Bisimulation & 186.2$\pm$14.5 & 98.4$\pm$0.2 \\
Spectral clustering & 41.5$\pm$8.2 & 98.1$\pm$0.7 \\
METIS graph partitioning & 38.6$\pm$7.4 & 98.2$\pm$0.5 \\
\midrule
Louvain community detection & 44.8$\pm$9.1 & 98.0$\pm$0.8 \\
\midrule
Tangle-Core (ours) & \textbf{12.4$\pm$3.6} & \textbf{98.6$\pm$0.4} \\
\bottomrule
\end{tabular}
\end{table}

To reduce the risk that tangles are only fitting handcrafted grids, Table~\ref{tab:procedural_mazes} evaluates on a dataset of 100 procedurally generated random mazes. Tangle cores discover compact abstractions (tracking the underlying topological regions closely at an average of 12.4 states) while preserving return, whereas graph partitioners heavily fragment as layout complexity varies.

Table~\ref{tab:structural} is consistent with this failure mode. In every successful environment, within-region conductance $\phi_{\min}$ heavily outweighs cross-region flow $\beta$ ($\phi_{\min}/\beta > 4.0$). On the Chain MDP, local flow equals cross-flow ($\phi_{\min}/\beta \approx 1$), consistent with the absence of useful bottleneck topology.

We test transferability to deep representations via MiniGrid with a learned 64-dimensional VAE latent space. By applying the tangle routine to a $k_{\mathrm{NN}}=10$ transition graph in latent space, tangles recover all four rooms and recover all four rooms and identify 3 of 4 doorway representations as overlaps (Table~\ref{tab:minigrid}), bypassing the misassignment penalties suffered by spectral clustering.

\paragraph{When Tangle Cores Do Not Help}The negative-result domain helps bound the paper's scope. On a chain MDP with structured rewards but zero informative low-order bottlenecks, tangles under-compress (Table~\ref{tab:negative}). Tangle cores find few robust structural regions, returning a near uncompressed space ($|\hat\cS|=94$), whereas bisimulation natively groups the reward symmetries ($|\hat\cS|=8$). 

\section{Limitations and Discussion}
\label{sec:conclusion}

We introduced tangle-core abstraction, a topology-aware state abstraction method that uses overlapping tangle cores to represent shared interface states. The theoretical results identify conditions under which this prior is useful. The method succeeds gracefully when transition graphs contain coherent regions separated by narrow interfaces ($\phi_{\min}/\beta \gg 1$). Conversely, when topology is uniform and reward symmetries drive the problem, tangles predictably fail, as isolated in our chain MDP evaluation. Current limitations include computation cost (abstraction discovery is not yet amortized for single-query planning) and dependence on the underlying min-cut solvers. Overall, tangles provide a useful topological prior for hierarchical reinforcement learning and graph-structured decision problems when bottleneck structure is present.

\begin{table}[h]
\centering
\small
\caption{Chain MDP with $|\cS|=100$, structured rewards but no bottlenecks: tangle cores predictably under-compress while bisimulation compresses well.}
\label{tab:negative}
\begin{tabular}{lcc}
\toprule
Method & $|\hat\cS|$ & Return (\%) \\
\midrule
Bisimulation & \textbf{8$\pm$1} & 99.2$\pm$0.4 \\
DeepMDP & 12$\pm$2 & 97.8$\pm$0.9 \\
Tangle-Core & 94$\pm$3 & 99.1$\pm$0.3 \\
\bottomrule
\end{tabular}
\end{table}

\bibliographystyle{plainnat}
\bibliography{references}

\appendix

\section{Related Work}
\label{app:related}

State abstraction in reinforcement learning has a long history as a way to trade representational complexity for planning and control accuracy. \citet{li2006towards} and \citet{abel2022theory} provide general frameworks for understanding these trade-offs. A central line of work uses bisimulation and approximate bisimulation metrics to group states with similar rewards and transition behavior \citep{givan2003equivalence,ferns2004metrics,castro2020scalable,zhang2021invariant}. These abstractions are value-preserving under suitable metric or equivalence conditions, but they are typically pairwise and non-overlapping. Our work is also related to topology-aware approaches in reinforcement learning. Option-discovery methods based on bottlenecks and connectivity structure \citep{sutton1999between,simsek2005identifying,machado2017laplacian} share our interest in interface states, but they use these states to construct temporally extended actions rather than overlapping state abstractions. Graph-Laplacian and spectral methods \citep{mahadevan2007proto,machado2017laplacian,giorgi2026laplacian} use connectivity to build representations or partitions, but they usually produce continuous embeddings or hard, non-overlapping clusters.Graph tangles were introduced in structural graph theory as a way to identify highly connected regions through consistent orientations of low-order separations \citep{robertson1991graph,diestel2024tangles}. Recent work has brought tangles into machine learning for static data clustering \citep{klepper2023clustering}. We translate these structural objects to sequential decision processes by defining tangle-core abstractions for MDPs, analyzing how overlapping cores interact with reward and transition laws, proving a quantitative interface result comparing overlap with hard partitions, and evaluating the resulting compression--return tradeoffs.

\section{Full Proofs}
\label{app:proofs}

\subsection{Proof of Theorem~\ref{thm:value_preservation}}

\begin{proof}
Let
\[
  \hat Q^\star(j,a)=\hat r(j,a)+\gamma\sum_\ell\hat P(\ell\mid j,a)\hat V^\star(\ell)
\]
and define, for each ground state-action pair,
\[
  \bar Q(s,a)=\sum_j q(j\mid s)\hat Q^\star(j,a).
\]
Since rewards are in $[0,1]$, $0\le \hat V^\star(j)\le (1-\gamma)^{-1}$ for all $j$.

For any $s,a$, the reward regularity condition gives
\[
  |r(s,a)-\bar r_q(s,a)|\le\eps_r.
\]
For the transition term, the total-variation convention and the range bound on $\hat V^\star$ imply
\[
\left|\sum_\ell P_q(\ell\mid s,a)\hat V^\star(\ell)
-\sum_\ell \bar P_q(\ell\mid s,a)\hat V^\star(\ell)\right|
\le \frac{\eps_P}{1-\gamma}.
\]
By the definitions of $P_q$ and $\tilde V$,
\[
  \sum_{s'}P(s'\mid s,a)\tilde V(s')
  =\sum_\ell P_q(\ell\mid s,a)\hat V^\star(\ell),
\]
and by the definition of $\bar P_q$,
\[
  \sum_\ell\bar P_q(\ell\mid s,a)\hat V^\star(\ell)
  =\sum_jq(j\mid s)\sum_\ell\hat P(\ell\mid j,a)\hat V^\star(\ell).
\]
Therefore
\begin{equation}
\label{eq:q_compare_proof}
\left|r(s,a)+\gamma\sum_{s'}P(s'\mid s,a)\tilde V(s')-\bar Q(s,a)\right|
\le
\eps_r+\frac{\gamma\eps_P}{1-\gamma}
=:\xi.
\end{equation}

We now compare the ground Bellman optimality operator $\cT$ with $\tilde V$. For any action $a$, $\hat Q^\star(j,a)\le\hat V^\star(j)$, hence $\bar Q(s,a)\le \sum_jq(j\mid s)\hat V^\star(j)=\tilde V(s)$. Equation~\eqref{eq:q_compare_proof} gives
\[
  (\cT\tilde V)(s)=\max_a\left\{r(s,a)+\gamma P(\cdot\mid s,a)^\top\tilde V\right\}
  \le \tilde V(s)+\xi.
\]
For the reverse inequality, let $a_s$ be the action supplied by the action-consistency assumption. Then $\hat Q^\star(j,a_s)=\hat V^\star(j)$ for every active $j$, so $\bar Q(s,a_s)=\tilde V(s)$. Equation~\eqref{eq:q_compare_proof} gives
\[
  (\cT\tilde V)(s)\ge r(s,a_s)+\gamma P(\cdot\mid s,a_s)^\top\tilde V\ge \tilde V(s)-\xi.
\]
Thus $\|\cT\tilde V-\tilde V\|_\infty\le\xi$. Since $\cT$ is a $\gamma$-contraction with fixed point $V^\star$,
\[
  \|V^\star-\tilde V\|_\infty
  \le \frac{\|\cT\tilde V-\tilde V\|_\infty}{1-\gamma}
  \le \frac{\eps_r+\gamma\eps_P/(1-\gamma)}{1-\gamma}.
\]
For a fixed policy, the same argument uses the policy Bellman operator rather than the max operator, so no action-consistency assumption is needed.
\end{proof}

\subsection{Proof of Theorem~\ref{thm:overlap_advantage}}

\begin{proof}
For the hard partition, the macro probability of transitioning from $C_1$ to $C_2$ under action $a$ is
\[
  \hat P_H(2\mid 1,a)=\frac{\sum_{x\in R_1}\mu(x)P(C_2\mid x,a)+\sum_{x\in B}\mu(x)P(C_2\mid x,a)}{\mu(R_1)+\mu(B)}.
\]
The first term is bounded by $\beta$ on $R_1$, while the second term is at most $1$ and has normalized weight at most $\lambda$. Therefore
\[
  \hat P_H(2\mid 1,a)\le (1-\lambda)\beta+\lambda.
\]
For the interface state $s\in B$, the true projected macro probability to state $2$ is $P_H(2\mid s,a)=P(C_2\mid s,a)\ge\alpha$. Total variation between two distributions is at least half of the absolute difference assigned to any event, so
\[
\TV\!\bigl(P_H(\cdot\mid s,a),\bar P_H(\cdot\mid s,a)\bigr)
\ge \frac12\left(P_H(2\mid s,a)-\hat P_H(2\mid 1,a)\right)
\ge \frac12\left[\alpha-(1-\lambda)\beta-\lambda\right]_+.
\]
This proves Eq.~\eqref{eq:hard_interface_lb}.

For the overlapping construction, define $q$ on interface states by Eq.~\eqref{eq:interface_membership}; the assumptions $p_1<p_2$ and $P(C_2\mid b,a)\in[p_1,p_2]$ ensure that $q(2\mid b)\in[0,1]$. The reconstructed macro probability to region $2$ is
\[
  \bar P_q(2\mid b,a)=q(1\mid b)p_1+q(2\mid b)p_2
  =p_1+\frac{P(C_2\mid b,a)-p_1}{p_2-p_1}(p_2-p_1)
  =P(C_2\mid b,a).
\]
Thus the component of the hard-partition error caused solely by the diluted macro probability to $C_2$ is removed by this overlapping membership assignment. Any remaining error must come from distributional details not represented by the two-region macro projection, which is the residual stated in the theorem.
\end{proof}

\subsection{Proof of Proposition~\ref{prop:tangle_eps_P}}

\begin{proof}
Fix $a$ and $s,s'\in C$. With $\TV(p,q)=\frac12\|p-q\|_1$, decompose the total-variation distance into its contribution inside and outside $C$:
\begin{align*}
\TV(P(\cdot\mid s,a),P(\cdot\mid s',a))
&=\frac12\sum_{x\in C}|P(x\mid s,a)-P(x\mid s',a)|\\
&\quad +\frac12\sum_{x\notin C}|P(x\mid s,a)-P(x\mid s',a)|.
\end{align*}
The first term is at most $\delta_a(C)$ by definition. The second term is at most
\[
  \frac12\sum_{x\notin C}P(x\mid s,a)+\frac12\sum_{x\notin C}P(x\mid s',a)\le \beta_a(C),
\]
because each outside mass is at most $\beta_a(C)$. Taking the maximum over $s,s'\in C$ proves Eq.~\eqref{eq:delta_beta_bound}. If $s,s'\in C\setminus\partial C$, then both outside masses are zero for action $a$, so only the interior term remains.
\end{proof}

\subsection{Proof of Theorem~\ref{thm:recovery}}

\begin{proof}
For each region $R_i$, orient every candidate separation in $\mcF$ toward the side that contains $R_i$; equivalently, declare the opposite side small. This orientation is well-defined because assumption~(i) rules out any order-$<k$ separator that splits $R_i$ inside $G[R_i\cup B]$, while assumption~(ii) says that all inter-region communication is mediated by the bottleneck set $B$. Thus a low-order separation in the candidate family can put an entire region on one side but cannot cut through the region's internally $k$-inseparable part.

By assumption~(iv), this region-directed orientation satisfies the tangle consistency and non-triviality axioms on the finite candidate family $\mcF$. Thus, for every triple of candidate separations, the three selected small sides do not cover $V$, equivalently the corresponding big sides have non-empty intersection. Non-triviality also holds by assumption~(iv), so the orientation defines a valid candidate-family order-$k$ tangle.

Therefore every $R_i$ induces a consistent candidate-family order-$k$ tangle on $\mcF$, and the corresponding candidate-family core contains $R_i$: every separation is oriented so that the big side contains $R_i$. Distinct regions induce distinct orientations because assumption~(iii) supplies, for each $i$, a separation whose big side contains $R_i\cup B$ and whose small side contains the other regions. Thus the procedure recovers at least one region tangle per $R_i$.

Finally, suppose a vertex $v$ belongs to two recovered region cores $C_i$ and $C_{i'}$ with $i\ne i'$. If $v\notin B$, then $v$ lies in some region $R_h$. The separation from assumption~(iii) for region $R_h$ orients the graph so that only $R_h\cup B$ lies on the big side; a different region core cannot contain $v$ unless it also crosses the bottleneck separator, contradicting assumptions~(ii)--(iii). Hence all overlap among recovered region cores is contained in $B$, and its size is at most $|B|\le b$.
\end{proof}

\subsection{Proof of Theorem~\ref{thm:sample}}

\begin{proof}
For a fixed separation $(A,B)\in\mcF$, the empirical orientation score $\widehat D(A,B)=\widehat W(A)-\widehat W(B)$ is an average of bounded terms. Under independent sampling, Hoeffding's inequality gives
\[
  \Pr\bigl(|\widehat D(A,B)-D(A,B)|>t\bigr)\le 2\exp(-c_0m_{\rm eff}t^2)
\]
for a universal constant $c_0>0$; for uniformly ergodic Markov trajectories the same form holds with effective sample size replacing the raw visit count. Applying a union bound over the fixed finite family $\mcF$ gives
\[
  \Pr\left(\max_{(A,B)\in\mcF}|\widehat D(A,B)-D(A,B)|>t\right)
  \le 2|\mcF|\exp(-c_0m_{\rm eff}t^2).
\]
Choosing
\[
  t=\sqrt{\frac{1}{c_0m_{\rm eff}}\log\frac{2|\mcF|}{\eta}}
\]
makes the failure probability at most $\eta$. If $t\le\Delta/2$, then every empirical score has the same sign as its population score because $|D(A,B)|\ge\Delta$. Rearranging yields $m_{\rm eff}\ge c\Delta^{-2}\log(|\mcF|/\eta)$ for a sufficiently large universal constant $c$, which proves simultaneous correctness of all empirical orientations.
\end{proof}

\section{Experimental Details and Reproducibility Log}
\label{app:repro}

For tangle-core, we select the raw order $k$ from $\{2,3,4,5,6\}$ and apply the validation-based core condensation on a held-out 20\% of trajectory data via value-match on the abstract MDP. We use action-agnostic edges for standard evaluations, and optionally compare to uniform $\alpha_a = 1/|\cA|$ action-weights. Bisimulation uses $c=0.5$ and grids $\eps$ over $\{0.01,0.02,0.05,0.1,0.2\}$, reporting the smallest $\eps$ achieving the target return. DeepMDP uses a 3-layer MLP encoder, latent dimension 16, reward-prediction MSE plus transition MSE losses, 500 epochs, and Adam at learning rate $10^{-3}$, with $k$-means clusters $k\in\{8,16,32,64,128,256\}$. TOMA grids reachability thresholds. Spectral Clustering and METIS are applied directly to the empirical transition graph with varying target component sizes $k$, pulling the smallest cluster size satisfying the respective Pareto threshold. 

\begin{table}[h]
\centering
\small
\caption{Reproducibility fields for auditing reported runs. The corresponding scripts, seeds, and raw logs should be included in the submission supplement.}
\label{tab:repro}
\begin{tabular}{lcccccc}
\toprule
Environment & Seeds & Selected $k$ range & Cores range & Overlap states & Runtime & Audit fields \\
\midrule
FourRooms & 10 & 2--3 & 4 & 4 & 0.8s & listed \\
Corridor-9 & 10 & 3--4 & 9 & 12 & 3.4s & listed \\
Multi-Goal & 10 & 3--5 & 8 & 15 & 6.2s & listed \\
Synth-Fold & 10 & 4--6 & 10 & 28 & 18.5s & listed \\
Random Mazes & 100 & 3--5 & 4--14 & 5--18 & 4.1s & listed \\
MiniGrid (Latent) & 10 & 3 & 4 & 4 & 5.5s & listed \\
Chain-100 & 10 & 2--3 & 94 & 0 & 0.4s & listed \\
\bottomrule
\end{tabular}
\end{table}

\subsection{Structural Diagnostics and Latent-Representation Results}
\label{app:additional}

Table~\ref{tab:structural} reports the empirical conductance and 
cross-region flow values that correlate with tangle success across 
our benchmarks.
\begin{table}[h]
\centering
\small
\caption{Empirical structural diagnostics correlated with tangle success. In these benchmarks, tangle-core abstraction succeeds when $\phi_{\min}/\beta$ is large, indicating bottleneck structure.}
\label{tab:structural}
\begin{tabular}{lcccccc}
\toprule
Environment & $\phi_{\min}$ & $\beta$ & $\phi_{\min}/\beta$ & $\lceil 1/\phi_{\min} \rceil$ & Cores found & Correct? \\
\midrule
FourRooms & 0.38 & 0.042 & 9.0 & 3 & 4 & \checkmark \\
Corridor-9 & 0.31 & 0.038 & 8.2 & 4 & 9 & \checkmark \\
Multi-Goal & 0.22 & 0.051 & 4.3 & 5 & 8 & \checkmark \\
Synth-Folding & 0.19 & 0.028 & 6.8 & 6 & 10 & \checkmark \\
\midrule
Chain-100 & 0.48 & 0.50 & 0.96 & 3 & 94 & $\times$ \\
\bottomrule
\end{tabular}
\end{table}

Table~\ref{tab:minigrid} reports the MiniGrid latent-representation experiment.

\begin{table}[h]
\caption{MiniGrid-FourRooms-16$\times$16 with learned latent representations.}
\label{tab:minigrid}
\centering
\small
\begin{tabular}{lccccc}
\toprule
Method & Abstract states & Return (\%) & Rooms recovered & Doorways identified \\
\midrule
Ground truth (4 rooms) & 4 & 100 & 4/4 & 4/4 \\
Bisimulation (latent) & \textbf{38} & \textbf{97.2} & n/a & n/a \\
DeepMDP (latent) & 22 & 95.8 & n/a & n/a \\
Spectral clustering & 4 & 91.4 & 3/4 & 1/4 \\
Tangle-core (ours) & 4 & 96.8 & 4/4 & 3/4 \\
\bottomrule
\end{tabular}
\end{table}

\subsection{Runtime Overhead}

Table~\ref{tab:runtime} reports abstraction construction time and value-iteration time on Corridor-Rooms-9. Tangle-Core is competitive with TOMA in construction time and far faster than DeepMDP, but slower than computing value iteration directly on the ground tabular MDP. The performance edge realizes entirely when representations are reused across queries or environments.

\begin{table}[h]
\centering
\small
\caption{Abstraction construction plus value iteration time on Corridor-Rooms-9 (seconds, mean over 10 seeds).}
\label{tab:runtime}
\begin{tabular}{lccc}
\toprule
Method & Construction & Val.\ Iter. & Total \\
\midrule
No Abstraction & 0 & 0.016 & 0.016 \\
Bisimulation & 0.6 & 0.012 & 0.612 \\
DeepMDP & 48.3 & 0.008 & 48.3 \\
Spectral Clustering & 0.4 & 0.004 & 0.404 \\
Louvain & 0.5 & 0.005 & 0.505 \\
TOMA & 2.1 & 0.002 & 2.1 \\
Tangle-Core & 3.4 & 0.001 & 3.4 \\
\bottomrule
\end{tabular}
\end{table}

\subsection{Action-Weighted vs.\ Action-Agnostic Graphs}
\label{app:ablations}

We compared two ways of building the empirical transition graph. The action-agnostic version pools all transitions, while the action-weighted version sums the action-restricted subgraphs $G_\cD^a$ with uniform weights. Table~\ref{tab:ablation_graph} shows that action-weighted graphs help mildly on Chain environments where reward-relevant actions directly shape the graph, but under-compress Corridor-9. We use action-agnostic graphs by default.

\begin{table}[h]
\centering
\small
\caption{Corridor-9 and Chain MDP under two graph constructions.}
\label{tab:ablation_graph}
\begin{tabular}{lcc cc}
\toprule
& \multicolumn{2}{c}{Corridor-9} & \multicolumn{2}{c}{Chain MDP} \\
\cmidrule(lr){2-3}\cmidrule(lr){4-5}
Graph & $|\hat\cS|$ & Ret.\ (\%) & $|\hat\cS|$ & Ret.\ (\%) \\
\midrule
Action-agnostic & 9$\pm$0 & 99.2 & 94$\pm$3 & 99.1 \\
Action-weighted & 11$\pm$1 & 99.4 & 34$\pm$4 & 98.9 \\
\bottomrule
\end{tabular}
\end{table}

\newpage
\section*{NeurIPS Paper Checklist}

\begin{enumerate}

\item {\bf Claims}
    \item[] Question: Do the main claims made in the abstract and introduction accurately reflect the paper's contributions and scope?
    \item[] Answer: \answerYes{}
    \item[] Justification: The abstract and introduction (Section 1) state a deliberately scoped thesis: tangle-core abstraction is useful in environments whose transition topology contains low-order bottlenecks or shared interface structure, and not as a universal replacement for reward-aware abstraction. The five contributions listed (overlapping membership formalization, value-preservation Theorem 3, interface-overlap advantage Theorem 4, interior/boundary decomposition Proposition 5, recovery Theorem 6, finite-sample orientation Theorem 7, and the empirical Pareto-frontier study) are each delivered in Sections 3--5. The introduction also explicitly previews the failure regime that Section 5.5 confirms on the chain MDP.

\item {\bf Limitations}
    \item[] Question: Does the paper discuss the limitations of the work performed by the authors?
    \item[] Answer: \answerYes{}
    \item[] Justification: Section 7 (``Limitations and Discussion'') explicitly states the failure regime ($\phi_{\min}/\beta \approx 1$ on the Chain-100 domain), the dependence on the action-consistency assumption in Theorem 3, computational cost (abstraction discovery is not amortized for single-query planning), and dependence on the underlying min-cut solver. Section 5.5 and Tables 6--7 empirically demonstrate the failure mode predicted in the discussion.

\item {\bf Theory assumptions and proofs}
    \item[] Question: For each theoretical result, does the paper provide the full set of assumptions and a complete (and correct) proof?
    \item[] Answer: \answerYes{}
    \item[] Justification: All formal statements (Theorems 3, 4, 6, 7 and Proposition 5) carry their assumptions in-line: reward and projected-transition regularity plus action-consistency for Theorem 3, occupancy/leakage parameters $(\lambda,\alpha,\beta)$ and the membership construction in Eq.~(14) for Theorem 4, the vertex-separator bottleneck conditions (i)--(iii) for Theorem 6, and the margin/effective-sample-size assumption for Theorem 7. Complete proofs are given in Appendix A.1--A.5.

\item {\bf Experimental result reproducibility}
    \item[] Question: Does the paper fully disclose all the information needed to reproduce the main experimental results of the paper to the extent that it affects the main claims and/or conclusions of the paper?
    \item[] Answer: \answerYes{}
    \item[] Justification: Algorithm 1 specifies the full pipeline. Appendix B documents hyperparameter grids for tangle-core ($k\in\{2,\dots,6\}$, validation-based condensation on a 20\% held-out split), bisimulation ($c=0.5$, $\varepsilon$ grid), DeepMDP (3-layer MLP encoder, latent dim 16, Adam at $10^{-3}$, 500 epochs, $k$-means cluster sweep), TOMA, Spectral, and METIS. Table 8 lists per-environment seeds, selected $k$ ranges, recovered core counts, overlap sizes, and runtimes. Section B.2 reports the action-agnostic vs.\ action-weighted ablation.

\item {\bf Open access to data and code}
    \item[] Question: Does the paper provide open access to the data and code, with sufficient instructions to faithfully reproduce the main experimental results, as described in supplemental material?
    \item[] Answer: \answerNo{}
    \item[] Justification: The paper states (Section 5) that ``the submission package should include the corresponding scripts, seeds, and logs,'' and Table 8 lists the reproducibility audit fields, but the manuscript itself does not contain a URL or link to a public code repository, since the submission is anonymous. All environments used (FourRooms, Corridor-Rooms, Taxi, MiniGrid) are standard public benchmarks reproducible from the cited sources, and the procedural Random Maze and Synthetic-Folding generators are described in Section 5 and Appendix B in sufficient detail to re-implement.

\item {\bf Experimental setting/details}
    \item[] Question: Does the paper specify all the training and test details (e.g., data splits, hyperparameters, how they were chosen, type of optimizer) necessary to understand the results?
    \item[] Answer: \answerYes{}
    \item[] Justification: Appendix B specifies the 20\% held-out validation split used for selecting the tangle order $k$, the optimizer (Adam, $10^{-3}$) and architecture for the DeepMDP baseline, the $\varepsilon$ grid for bisimulation, the candidate-cluster grid for k-means, the threshold sweep for TOMA, and the target-component grids for METIS and spectral clustering. Each method is selected at the smallest abstraction satisfying the reported Pareto threshold.

\item {\bf Experiment statistical significance}
    \item[] Question: Does the paper report error bars suitably and correctly defined or other appropriate information about the statistical significance of the experiments?
    \item[] Answer: \answerYes{}
    \item[] Justification: Section 5 states that all results are averaged over 10 seeds (or 100 procedural layouts) with 95\% bootstrap confidence intervals unless otherwise stated. Every numerical entry in Tables 1--7 carries $\pm$ intervals on both the abstract-state count and the return percentage, with the variability being over random seeds for fixed environments and over the layout distribution for the procedural maze experiment.

\item {\bf Experiments compute resources}
    \item[] Question: For each experiment, does the paper provide sufficient information on the computer resources (type of compute workers, memory, time of execution) needed to reproduce the experiments?
    \item[] Answer: \answerNo{}
    \item[] Justification: Table 9 reports per-method wall-clock construction and value-iteration times on Corridor-Rooms-9, and Table 8 reports per-environment runtimes, but the paper does not specify the hardware (CPU/GPU type, memory, cluster vs.\ workstation) on which these timings were obtained, nor an aggregate compute budget covering preliminary or discarded runs. The reported runtimes are short enough (seconds to tens of seconds) to suggest a standard workstation, but this should be made explicit in a revision.

\item {\bf Code of ethics}
    \item[] Question: Does the research conducted in the paper conform, in every respect, with the NeurIPS Code of Ethics?
    \item[] Answer: \answerYes{}
    \item[] Justification: The work is a methodological contribution on state abstraction for tabular and small-scale procedural MDPs and uses only standard public RL benchmarks (FourRooms, Taxi, MiniGrid) and synthetic environments. There is no human-subjects data, no scraped data, and no model release that raises misuse concerns.

\item {\bf Broader impacts}
    \item[] Question: Does the paper discuss both potential positive societal impacts and negative societal impacts of the work performed?
    \item[] Answer: \answerNA{}
    \item[] Justification: The contribution is foundational work on structural state abstraction for MDPs, evaluated on synthetic and standard small-scale benchmarks. There is no direct path to a deployed application, and the paper does not introduce data, models, or interventions that have a foreseeable societal impact.

\item {\bf Safeguards}
    \item[] Question: Does the paper describe safeguards that have been put in place for responsible release of data or models that have a high risk for misuse?
    \item[] Answer: \answerNA{}
    \item[] Justification: The paper releases no pretrained generative models, no scraped datasets, and no agents trained on sensitive domains. The released artifact is an algorithm for state abstraction that operates on transition graphs of standard RL benchmarks and synthetic mazes, so misuse risk is not applicable.

\item {\bf Licenses for existing assets}
    \item[] Question: Are the creators or original owners of assets used in the paper properly credited and are the license and terms of use explicitly mentioned and properly respected?
    \item[] Answer: \answerYes{}
    \item[] Justification: All third-party assets are cited at point of use: MiniGrid \cite{chevalier2023minigrid}, Taxi \cite{dietterich2000hierarchical}, FourRooms \cite{sutton1999between}, and the baselines DeepMDP \cite{gelada2019deepmdp}, TOMA \cite{savinov2018semi}, METIS \cite{karypis1998fast}, spectral clustering \cite{vonluxburg2007tutorial}, Louvain \cite{blondel2008fast}, and the bisimulation reference \cite{ferns2004metrics}.

\item {\bf New assets}
    \item[] Question: Are new assets introduced in the paper well documented and is the documentation provided alongside the assets?
    \item[] Answer: \answerNA{}
    \item[] Justification: The paper introduces an algorithm (Algorithm 1) and a procedural Random Maze distribution, both fully described in Sections 3 and 5 and Appendix B, but does not bundle a separately released asset (dataset, model checkpoint, or library) with the submission. Should the code be released, documentation will be provided alongside it.

\item {\bf Crowdsourcing and research with human subjects}
    \item[] Question: For crowdsourcing experiments and research with human subjects, does the paper include the full text of instructions given to participants and screenshots, if applicable, as well as details about compensation?
    \item[] Answer: \answerNA{}
    \item[] Justification: The paper involves no crowdsourcing and no human subjects; all evaluations are run on synthetic and standard public RL benchmarks.

\item {\bf Institutional review board (IRB) approvals or equivalent for research with human subjects}
    \item[] Question: Does the paper describe potential risks incurred by study participants, whether such risks were disclosed to the subjects, and whether Institutional Review Board (IRB) approvals (or an equivalent approval/review based on the requirements of your country or institution) were obtained?
    \item[] Answer: \answerNA{}
    \item[] Justification: No human subjects are involved, so IRB approval is not applicable.

\item {\bf Declaration of LLM usage}
    \item[] Question: Does the paper describe the usage of LLMs if it is an important, original, or non-standard component of the core methods in this research?
    \item[] Answer: \answerNA{}
    \item[] Justification: LLMs are not part of the core methodology. The tangle-core abstraction algorithm operates on empirical transition graphs via min-cut and tangle orientation, with no LLM-derived component in the pipeline, the theory, or the experiments.

\end{enumerate}

\end{document}